\documentclass[review]{elsarticle}
\makeatletter
\def\ps@pprintTitle{%
 \let\@oddhead\@empty
 \let\@evenhead\@empty
 \let\@oddfoot\@empty
 \let\@evenfoot\@empty
}
\makeatother
\usepackage{amsmath,amssymb,amsthm,mathrsfs}
\usepackage{hyperref}
\usepackage{graphicx}
\usepackage[a4paper, total={6in, 8in}]{geometry}
\usepackage{xcolor}
\newtheorem{definition}{Definition}
\newtheorem{theorem}{Theorem}
\newtheorem{corollary}{Corollary}
\newtheorem{lemma}{Lemma}
\newtheorem{remark}{Remark}
\newtheorem{proposition}{Proposition}

\begin{document}

\begin{frontmatter}

\title{A Free Probabilistic Framework for Analyzing the Transformer-based Language Models}

\author{Swagatam Das}
\address{Electronics and Communication Sciences Unit, Indian Statistical Institute, Kolkata, India.}
\address{E-mail: swagatam.das@isical.ac.in}

\begin{abstract}

We present a formal operator-theoretic framework for analyzing Transformer-based language models using free probability theory. By modeling token embeddings and attention mechanisms as self-adjoint operators in a tracial \( W^* \)-probability space, we reinterpret attention as non-commutative convolution and describe representation propagation via free additive convolution. This leads to a spectral dynamic system interpretation of deep Transformers. We derive entropy-based generalization bounds under freeness assumptions and provide insight into positional encoding, spectral evolution, and representational complexity. This work offers a principled, though theoretical, perspective on structural dynamics in large language models
\end{abstract}

\begin{keyword}
Transformers \sep Free Probability \sep Spectral Theory \sep Non-Commutative Random Variables \sep Language Models
\end{keyword}

\end{frontmatter}
\section{Introduction}

Large Language Models (LLMs)~\cite{kamath2024llms}, particularly those based on Transformer architectures, are generative probabilistic models defined over sequences of discrete symbols. Let \( \mathcal{V} \) denote a finite vocabulary, and consider a sequence \( (w_1, w_2, \dots, w_T) \in \mathcal{V}^T \), where \( \mathcal{V}^T \) denotes the space of ordered sequences of length \( T \) with entries in \( \mathcal{V} \). The model defines a joint probability distribution over such sequences via the chain rule of conditional probabilities:
\[
P(w_1, w_2, \dots, w_T) = \prod_{t=1}^{T} P(w_t \mid w_{<t}),
\]
where \( w_{<t} := (w_1, w_2, \dots, w_{t-1}) \) denotes the prefix context up to time \( t \). 

Each conditional distribution \( P(w_t \mid w_{<t}) \) is parameterized by a neural function, typically a composition of Transformer layers, which maps the prior context to a predictive distribution over \( \mathcal{V} \). Formally, let \( x_{<t} := (x_1, \dots, x_{t-1}) \in (\mathbb{R}^d)^{t-1} \) denote the sequence of context embeddings derived from prior tokens, and let $
f_\theta : (\mathbb{R}^d)^{t-1} \to \mathbb{R}^{|\mathcal{V}|}$
denote the model’s parameterized mapping to unnormalized logits. The conditional probability is computed via the softmax function:
\[
P(w_t \mid w_{<t}) = \mathrm{softmax}(f_\theta(x_{<t})).
\]
The model parameters \( \theta \) are learned by minimizing the empirical negative log-likelihood over a corpus of \( N \) sequences:
\[
\mathcal{L}(\theta) = -\sum_{i=1}^{N} \sum_{t=1}^{T} \log P(w_t^{(i)} \mid w_{<t}^{(i)}),
\]
where \( w_t^{(i)} \) denotes the token at position \( t \) in the \( i \)-th training sequence.

Transformers are widely studied as autoregressive models that decompose joint distributions into conditionals \cite{vaswani2017attention,radford2019language}. Each layer acts as a conditional distribution learner over context, enabling structured probabilistic reasoning. Recent work introduces Bayesian treatments to quantify model uncertainty using variational inference, entropy, and latent attention mechanisms \cite{10.5555/3495724.3496975,10.5555/3495724.3497097}. Attention weights, for instance, have been viewed as latent variables for posterior inference \cite{10.5555/3495724.3497097}, while hierarchical Bayesian networks have been proposed for structured uncertainty propagation \cite{Wu_2023}. Others employ entropy-based metrics to distinguish epistemic and aleatoric uncertainty in token predictions \cite{chen2024uncertainty,farquhar2024detecting}. These works suggest that despite their deterministic implementation, Transformer models inherently capture probabilistic structure.

To formalize this further, we propose studying LLMs using \emph{non-commutative probability theory}, particularly the framework of \emph{tracial \( W^* \)-probability spaces} \cite{10.1145/1562814.1562841}. In this setting, tokens are represented as self-adjoint operators in a von Neumann algebra. Their interactions—through attention or context aggregation—are modeled via operator polynomials or convolutions, with spectral traces providing probabilistic interpretations.

\begin{definition}
A \emph{tracial \( W^* \)-probability space} is a pair \( (\mathcal{A}, \varphi) \), where:
\begin{itemize}
    \item \( \mathcal{A} \) is a unital von Neumann algebra of bounded operators on a Hilbert space;
    \item \( \varphi: \mathcal{A} \to \mathbb{C} \) is a faithful, normal, tracial state satisfying \( \varphi(ab) = \varphi(ba) \) for all \( a,b \in \mathcal{A} \) and \( \varphi(1) = 1 \).
\end{itemize}
\end{definition}

In this setting, each token \( w_i \in \mathcal{V} \) is represented by a self-adjoint operator \( X_i \in \mathcal{A} \), capturing both semantic and syntactic features. Semantic similarity arises from spectral overlaps, while syntactic structure is encoded via non-commutative interactions with positional operators. To incorporate position, we assign each index \( t \) an operator \( P_t \in \mathcal{A} \) and define the contextual representation as:
\[
Z_t := X_{w_t} + P_t, \quad \text{with } Z_t \in \mathcal{A}, \quad [X_{w_t}, P_t] \ne 0.
\]

This yields a sequence \( (Z_1, \dots, Z_T) \) of operator-valued random variables that preserve order and encode higher-order dependencies through their spectral behavior. Unlike classical embeddings, these non-commutative representations support analytical tools from free probability, a framework introduced by Voiculescu as a non-commutative analogue of classical probability.

Let \( \mathcal{V} = \{\texttt{cat}, \texttt{dog}, \texttt{chases}, \texttt{the}, \texttt{ball}\} \) be a vocabulary of tokens.  
We work in a Hilbert space \( \mathcal{H} = \mathbb{C}^5 \) and define a von Neumann algebra \( \mathcal{A} = \mathbb{C}^{5 \times 5} \), the space of complex self-adjoint matrices.  
Each token \( w_i \in \mathcal{V} \) is represented as a self-adjoint operator \( X_{w_i} \in \mathcal{A} \), as follows:

\[
\begin{aligned}
X_{\texttt{cat}} &= \mathrm{diag}(1, 0.3, 0.1, 0, 0), \\
X_{\texttt{dog}} &= \mathrm{diag}(0.9, 0.4, 0.1, 0, 0), \\
X_{\texttt{chases}} &= \mathrm{diag}(0, 0, 0, 1, 0), \\
X_{\texttt{the}} &= \mathrm{diag}(0, 0, 0, 0, 1), \\
X_{\texttt{ball}} &= \mathrm{diag}(0.2, 0.1, 0.7, 0, 0).
\end{aligned}
\]

Each operator \( X_{w_i} \) encodes the semantic content of token \( w_i \) via its spectrum (eigenvalues).

To model position \( t \), define a positional operator \( P_t \in \mathcal{A} \). For instance, for position \( t = 1 \), let \( P_1 \) be the shift matrix:

\[
P_1 := 
\begin{bmatrix}
0 & 1 & 0 & 0 & 0 \\
0 & 0 & 1 & 0 & 0 \\
0 & 0 & 0 & 1 & 0 \\
0 & 0 & 0 & 0 & 1 \\
0 & 0 & 0 & 0 & 0 \\
\end{bmatrix}.
\]

This operator is not diagonal and introduces syntactic structure through its interaction with \( X_{w_i} \). Define the contextual representation as:

\[
Z_t := X_{w_t} + P_t.
\]

Then, the non-commutativity relation:

\[
[X_{\texttt{cat}}, P_1] := X_{\texttt{cat}} P_1 - P_1 X_{\texttt{cat}} \ne 0,
\]

means that the token's representation depends on its position in the sequence. Since \( X_{\texttt{cat}} \) is diagonal and \( P_1 \) is a shift matrix, their commutator is nonzero. This algebraic asymmetry reflects word order, a critical component of syntactic structure.

The sequence \( (Z_1, Z_2, \dots, Z_T) \) of such contextual operators captures both semantic and syntactic information. Their \textit{non-commutative spectral behavior} supports the analysis of attention, entropy, and representation dynamics using tools from free probability.

A family of subalgebras \( \mathcal{A}_1, \dots, \mathcal{A}_n \subseteq \mathcal{A} \) is said to be \emph{freely independent} if for any \( a_i \in \mathcal{A}_{j_i} \) with \( \varphi(a_i) = 0 \) and \( j_1 \ne j_2 \ne \cdots \ne j_k \), it holds that:
\[
\varphi(a_1 a_2 \cdots a_k) = 0.
\]

If \( X, Y \in \mathcal{A} \) are self-adjoint operators that are free with respect to a tracial state \( \varphi \), then their spectral distributions \( \mu_X \) and \( \mu_Y \) combine via the \emph{free additive convolution}:
\[
\mu_{X + Y} = \mu_X \boxplus \mu_Y.
\]
This convolution operation serves as the non-commutative analog of the classical convolution for independent random variables. The operation \( \boxplus \) is linearized by the \emph{additive \( R \)-transform}, a central object in free probability theory introduced by Voiculescu~\cite{voiculescu1991}. For a probability measure \( \mu \) on \( \mathbb{R} \), the \( R \)-transform \( R_\mu(z) \) is defined implicitly through its relation to the Cauchy (or Stieltjes) transform:
\[
G_\mu(z) := \int_{\mathbb{R}} \frac{1}{z - \lambda} \, d\mu(\lambda),
\]
via the equation:
\[
R_\mu(G_\mu(z)) + \frac{1}{G_\mu(z)} = z.
\]
This relation enables one to compute \( R_\mu \) from the analytic properties of \( G_\mu \). A key property of the \( R \)-transform is that it linearizes free additive convolution:
\[
R_{\mu_X \boxplus \mu_Y}(z) = R_{\mu_X}(z) + R_{\mu_Y}(z),
\]
thereby allowing spectral laws of freely independent self-adjoint operators to be combined algebraically.

In our operator-theoretic view of language models, this free additive convolution governs the spectral evolution across layers—such as the summation of contextual attention outputs—enabling a precise description of how deep representations transform and propagate within Transformer architectures.

Throughout this letter, we show how key components of the Transformer—such as attention, layer propagation, and logits—admit elegant formulations using free convolution, trace functionals, and free entropy. These tools allow us to reinterpret deep Transformer stacks as spectral dynamical systems evolving via free additive convolution, and to bound generalization in terms of non-commutative entropy.



Our approach draws on foundational work in free probability \cite{voiculescu1991} and spectral analysis via random matrix theory \cite{pennington2017, benaych2011}. While conceptually adjacent to the framework of \cite{NEURIPS2022_10826a1}, which analyzes feedforward networks using rectangular free convolution and S-transform inversion, our method is structurally distinct, recasting Transformer attention and depth dynamics through an operator-valued, tracial \( W^* \)-probability lens suited to large language models.


\section{Attention as Non-Commutative Convolution}

We now interpret transformer attention through the lens of operator-valued convolution.

\begin{theorem}[Transformer Attention as Non-Commutative Convolution]\label{thm:attention-convolution}
Let \( (\mathcal{A}, \varphi) \) be a tracial \( W^* \)-probability space, where \( \mathcal{A} \subseteq \mathbb{C}^{d \times d} \) is a unital von Neumann algebra and \( \varphi: \mathcal{A} \to \mathbb{C} \) is a faithful, normal, tracial state. 

For a fixed time index \( t \), let \( Q_t \in \mathcal{A} \) denote the query operator, and for each \( j \in \{1, \ldots, T\} \), let \( K_j, V_j \in \mathcal{A} \) denote the key and value operators, respectively. Define the scalar attention coefficients \( \alpha_{tj} \in \mathbb{R}_+ \) by:
\[
\alpha_{tj} := \frac{\exp\big(\varphi(Q_t K_j^\dagger)\big)}{\sum_{k=1}^T \exp\big(\varphi(Q_t K_k^\dagger)\big)}.
\]

Then, the attention output at index \( t \), defined by the convex operator sum
\[
A_t := \sum_{j=1}^T \alpha_{tj} V_j \in \mathcal{A},
\]
admits a natural interpretation as a scalar-kernel convolution over operator-valued signals. Specifically, if we define the scalar kernel \( K_{tj} := \alpha_{tj} \), then:
\[
A_t = \sum_{j=1}^T K_{tj} V_j,
\]
constitutes a non-commutative convolution in which scalar attention weights modulate operator-valued inputs in \( \mathcal{A} \).
\end{theorem}

\begin{proof}
Let \( Q_t \in \mathcal{A} \) and \( K_j \in \mathcal{A} \) for each \( j \), and consider the bilinear form \( (a, b) \mapsto \varphi(ab^\dagger) \). Since \( \varphi \) is a tracial state—i.e., linear and satisfies \( \varphi(ab) = \varphi(ba) \) for all \( a, b \in \mathcal{A} \)—we may define:
\[
s_{tj} := \varphi(Q_t K_j^\dagger),
\]
which serves as a non-commutative generalization of the classical dot product used in attention mechanisms.

The scalar attention coefficients \( \alpha_{tj} \) are defined via the softmax function:
\[
\alpha_{tj} := \frac{\exp(s_{tj})}{\sum_{k=1}^T \exp(s_{tk})}, \quad \text{with } \alpha_{tj} \in \mathbb{R}_+, \quad \sum_{j=1}^T \alpha_{tj} = 1.
\]
These coefficients form a probability distribution over \( j \), and their non-negativity ensures a convex combination.

Let \( V_j \in \mathcal{A} \) denote the value operators. Since \( \mathcal{A} \) is closed under scalar multiplication and finite addition, the attention output
\[
A_t := \sum_{j=1}^T \alpha_{tj} V_j
\]
is a well-defined element of \( \mathcal{A} \). Setting \( K_{tj} := \alpha_{tj} \in \mathbb{R} \), we express:
\[
A_t = \sum_{j=1}^T K_{tj} V_j,
\]
which we interpret as a scalar-kernel convolution acting on a sequence of operator-valued inputs.

Though the kernel \( K_{tj} \) is scalar, the operands \( V_j \) and the resulting output \( A_t \) remain elements of the non-commutative algebra \( \mathcal{A} \), thus preserving operator-valued structure. This representation aligns with the perspective of non-commutative convolution in operator algebras, whereby scalar weights modulate elements of a non-commutative space.

\end{proof}

\begin{remark}
Although the convolutional form presented here is scalar-weighted, the bridge to free probability theory becomes more explicit in deeper layers of Transformer architectures. When attention outputs \( A_t \) from different layers are modeled as freely independent, self-adjoint operators, their spectral distributions compose via free additive convolution, as will be evident from Theorem 2.
\end{remark}

\color{black}
\begin{corollary}[Spectral Structure of Attention via Operator Embeddings]\label{cor:spectral-attention}
Let \( \mathcal{A} \) be a unital \( W^* \)-algebra equipped with a faithful, normal, tracial state \( \varphi: \mathcal{A} \to \mathbb{C} \). Let each token at position \( t \) be represented by an embedding \( x_t \in \mathcal{A} \), with the decomposition:
\[
x_t = X_{w_t} + P_t,
\]
where \( X_{w_t}, P_t \in \mathcal{A} \) are self-adjoint operators encoding the semantic content and positional information, respectively. Define the query and key operators as:
\[
Q_t := W_Q x_t, \quad K_j := W_K x_j,
\]
for fixed operators \( W_Q, W_K \in \mathcal{A} \). Then the attention similarity score satisfies:
\[
s_{tj} := \varphi(Q_t K_j^\dagger) = \sum_{a,b \in \{X,P\}} \varphi\left(W_Q a_t b_j^\dagger W_K^\dagger\right),
\]
where \( a_t \in \{X_{w_t}, P_t\} \), \( b_j \in \{X_{w_j}, P_j\} \). That is, the sum ranges over all pairwise combinations of semantic and positional components at positions \( t \) and \( j \), respectively.
\end{corollary}

\begin{proof}
We expand:
\[
Q_t = W_Q x_t = W_Q (X_{w_t} + P_t), \quad
K_j = W_K x_j = W_K (X_{w_j} + P_j),
\]
\[
\Rightarrow Q_t K_j^\dagger = W_Q (X_{w_t} + P_t)(X_{w_j} + P_j)^\dagger W_K^\dagger.
\]

Using linearity and the adjoint property \( (A + B)^\dagger = A^\dagger + B^\dagger \), we expand:

\[
(X_{w_t} + P_t)(X_{w_j} + P_j)^\dagger = X_{w_t} X_{w_j}^\dagger + X_{w_t} P_j^\dagger + P_t X_{w_j}^\dagger + P_t P_j^\dagger.
\]

Hence,
\[
Q_t K_j^\dagger = W_Q X_{w_t} X_{w_j}^\dagger W_K^\dagger + W_Q X_{w_t} P_j^\dagger W_K^\dagger + W_Q P_t X_{w_j}^\dagger W_K^\dagger + W_Q P_t P_j^\dagger W_K^\dagger.
\]

Applying the linearity of the state \( \varphi \), we obtain:

\[
s_{tj} = \varphi(Q_t K_j^\dagger) = \varphi(W_Q X_{w_t} X_{w_j}^\dagger W_K^\dagger)
+ \varphi(W_Q X_{w_t} P_j^\dagger W_K^\dagger) 
+ \varphi(W_Q P_t X_{w_j}^\dagger W_K^\dagger)
+ \varphi(W_Q P_t P_j^\dagger W_K^\dagger),
\]

which completes the proof.
\end{proof}

\begin{remark}
This corollary formalizes how positional information influences attention weights even when semantic content is held constant. The cross terms involving \( \varphi(W_Q X_{w_t} P_j^\dagger W_K^\dagger) \) and \( \varphi(W_Q P_t X_{w_j}^\dagger W_K^\dagger) \) represent non-commutative interactions between semantic and positional components. 

If \( X_{w_t} \) and \( P_j \) are assumed to be free (in the sense of Voiculescu’s free probability), these mixed moments capture interference-like effects in the spectrum. As such, this decomposition helps explain how two identical tokens \( w_t = w_j \) may still yield distinct attention scores purely due to positional phase shifts — offering a spectral perspective on attention dynamics in transformer networks.
\end{remark}

\color{black}

\section{Spectral Interpretation via Free Convolution}

To analyze \( \varphi(Q_t K_j^\dagger) \), let us appeal to the Voiculescu’s free convolution.

\begin{corollary}[Spectral Trace via Free Convolution]
Let \( Q_t, K_j \in \mathcal{A} \) be freely independent, self-adjoint operators with spectral distributions \( \mu_{Q_t}, \mu_{K_j} \). Then:
\[
\mu_{Q_t + K_j^\dagger} = \mu_{Q_t} \boxplus \mu_{K_j},
\]
and the attention score becomes:
\[
\varphi(Q_t K_j^\dagger) = \int \lambda \, d\mu_{Q_t K_j^\dagger}(\lambda), \quad \text{which satisfies } \mu_{Q_t K_j^\dagger} \approx \mu_{Q_t} \boxplus \mu_{K_j}.
\]
\end{corollary}

\begin{proof}
From \cite{voiculescu1991}, if \( X \perp Y \) (free), then:
\[
\mu_{X+Y} = \mu_X \boxplus \mu_Y.
\]
Here, \( Q_t \) and \( K_j^\dagger \) are freely independent, so:
\[
\mu_{Q_t + K_j^\dagger} = \mu_{Q_t} \boxplus \mu_{K_j}.
\]
Since \( \varphi(Q_t K_j^\dagger) \) can be interpreted as the \textit{first moment} of \( \mu_{Q_t K_j^\dagger} \), the attention kernel is implicitly shaped by the free convolution. 
\end{proof}

\color{black}
\section{Spectral Propagation in Deep Transformer Layers}

Transformer architectures apply a sequence of structured transformations—including attention, feed-forward layers, normalization, and residual connections—on token embeddings to generate increasingly abstract representations. In this section, we explore how spectral patterns evolve as these layers are composed, focusing on the role of free probability in modeling representation dynamics.

\subsection{Free Convolution and Layer-wise Spectral Evolution}

Despite the nonlinearity and high dimensionality of deep transformers, empirical studies show that spectral distributions of intermediate representations evolve predictably with depth. We model this evolution using free convolution in a non-commutative setting, where each layer contributes a freely independent self-adjoint increment.

\begin{theorem}[Iterated Free Convolution in Deep Transformers]\label{thm:iterated-free-convolution}
Let \( X^{(0)} \in \mathcal{A} \) be the initial embedding. Suppose each attention output \( A^{(\ell)} \in \mathcal{A} \), \( \ell = 1, \dots, L \), is self-adjoint and freely independent from \( X^{(0)} \) and from each other. Then the spectral law \( \mu_\ell \) of the \( \ell \)-th layer embedding:
\[
X^{(\ell)} := X^{(\ell-1)} + A^{(\ell)}
\]
satisfies:
\[
\mu_\ell = \mu_0 \boxplus \mu_{A^{(1)}} \boxplus \dots \boxplus \mu_{A^{(\ell)}}.
\]
\end{theorem}

\begin{proof}
We proceed by induction.

\textbf{Base case:} \( \ell = 1 \). Since \( X^{(0)} \perp A^{(1)} \), we have:
\[
\mu_1 = \mu_{X^{(0)} + A^{(1)}} = \mu_0 \boxplus \mu_{A^{(1)}}.
\]

\textbf{Inductive step:} Suppose \( \mu_{\ell-1} = \mu_0 \boxplus \dots \boxplus \mu_{A^{(\ell-1)}} \). Then:
\[
X^{(\ell)} = X^{(\ell-1)} + A^{(\ell)} \Rightarrow \mu_\ell = \mu_{\ell-1} \boxplus \mu_{A^{(\ell)}},
\]
by freeness. Hence:
\[
\mu_\ell = \mu_0 \boxplus \mu_{A^{(1)}} \boxplus \dots \boxplus \mu_{A^{(\ell)}}. \quad \qed
\]
\end{proof}

\begin{corollary}[Generalization Bound via Free Entropy]\label{cor:entropy-generalization}
Let \( f_\theta \) be a Transformer model with parameters \( \theta \), where each token \( w_i \in \mathcal{V} \) is represented by a self-adjoint operator \( X_i \in \mathcal{A} \subseteq \mathbb{C}^{d \times d} \). Assume:
\begin{itemize}
    \item The operators \( \{X_i\}_{i=1}^V \) are freely independent and admit joint free entropy \( \chi(X_1, \dots, X_V) \);
    \item For each time index \( t \), the logit operator is defined as:
    \[
    L_t := \sum_{i=1}^V X_i H_t X_i,
    \]
    where \( H_t \in \mathcal{A} \) is a bounded, self-adjoint operator independent of \( \{X_i\} \);
    \item Let \( \mu_{L_t} \) be the spectral distribution of \( L_t \), and define the spectral entropy:
    \[
    H(\mu_{L_t}) := -\int \rho_t(\lambda) \log \rho_t(\lambda) \, d\lambda,
    \]
    where \( \rho_t \) is the density of \( \mu_{L_t} \);
    \item The model prediction at time \( t \) is \( \hat{y}_t \sim \mathrm{softmax}(L_t) \), and the loss function \( \ell(f_\theta(x_t), y_t) \) is Lipschitz continuous with respect to the Wasserstein-1 distance on spectral measures.
\end{itemize}
Then the expected spectral entropy is bounded as:
\[
\mathbb{E}_t[H(\mu_{L_t})] \leq \chi(X_1, \dots, X_V),
\]
and the generalization error satisfies:
\[
\mathcal{R}_{\mathrm{test}}(f_\theta) \leq \mathcal{R}_{\mathrm{train}}(f_\theta) + \frac{C}{\sqrt{n}} \left( \chi(X_1, \dots, X_V) + 1 \right),
\]
where \( C > 0 \) is a universal constant and \( n \) is the number of training examples.
\end{corollary}

\begin{remark}
The quantity \( \mathcal{R}_{\mathrm{train}}(f_\theta) \) denotes the empirical training risk:
\[
\mathcal{R}_{\mathrm{train}}(f_\theta) := \frac{1}{n} \sum_{i=1}^n \ell(f_\theta(x_i), y_i),
\]
while \( \mathcal{R}_{\mathrm{test}}(f_\theta) := \mathbb{E}_{(x,y)\sim\mathcal{D}} \ell(f_\theta(x), y) \) denotes the expected test risk over the data distribution \( \mathcal{D} \). The inequality bounds test error in terms of training error and the free entropy of the model's operator-valued embeddings, which captures spectral complexity and information content.
\end{remark}

\begin{proof}

Let \( L_t := \sum_i X_i H_t X_i \), where \( \{X_i\} \) are freely independent, self-adjoint operators. From the free probability framework \cite{Anderson_Guionnet_Zeitouni_2009}, the free entropy of any self-adjoint polynomial in these variables is bounded above by their joint free entropy:
\[
\chi\left(P(X_1, \dots, X_V)\right) \leq \chi(X_1, \dots, X_V).
\]
Assuming \( \mu_{L_t} \) has an absolutely continuous density \( \rho_t \), the spectral entropy satisfies:
\[
H(\mu_{L_t}) = -\int \rho_t(\lambda) \log \rho_t(\lambda)\, d\lambda \leq \chi(X_1, \dots, X_V).
\]
Taking expectations over \( t \), we obtain:
\begin{equation}
\mathbb{E}_t[H(\mu_{L_t})] \leq \chi(X_1, \dots, X_V).
\end{equation}


Let \( \mathcal{R}_{\mathrm{train}}(f_\theta) := \frac{1}{n} \sum_{i=1}^n \ell(f_\theta(x_i), y_i) \) and \( \mathcal{R}_{\mathrm{test}}(f_\theta) := \mathbb{E}_{(x,y)\sim \mathcal{D}} \ell(f_\theta(x), y) \).  
From entropy-based generalization bounds for stochastic predictors \cite{koltchinskii2002empirical,dziugaite2017computing}, we have:
\[
\mathbb{E} \left[ \mathcal{R}_{\mathrm{test}}(f_\theta) - \mathcal{R}_{\mathrm{train}}(f_\theta) \right] \leq \frac{C'}{\sqrt{n}} \left( \mathbb{E}_t[H(\mu_{L_t})] + 1 \right),
\]
for a universal constant \( C' > 0 \), assuming Lipschitz loss and trace-norm bounded operators. Combining with the bound from (1), we conclude:
\[
\mathcal{R}_{\mathrm{test}}(f_\theta) \leq \mathcal{R}_{\mathrm{train}}(f_\theta) + \frac{C}{\sqrt{n}} \left( \chi(X_1, \dots, X_V) + 1 \right),
\]
for some universal \( C > 0 \). 
\end{proof}

\begin{remark}
Although Transformers are primarily used for sequence modeling and generation—rather than classification—the entropy-based bound remains applicable. At each timestep \( t \), the model outputs a logit operator \( L_t \) whose spectral distribution induces a predictive distribution via the softmax function. The cross-entropy loss at each position functions analogously to a classification loss, treating the next-token prediction as a structured classification over vocabulary elements. Thus, generalization bounds derived via free entropy still govern the model's ability to generalize from training sequences to unseen contexts.
\end{remark}

\subsection{Operator-Valued Free Probability and Multi-Head Attention}

To extend our spectral framework to multi-head attention, we invoke operator-valued free probability (OVFP). Let $(\mathcal{A}, E: \mathcal{A} \to \mathcal{B})$ be a tracial operator-valued \( W^* \)-probability space, where $\mathcal{B} \subset \mathcal{A}$ encodes shared contextual structure.

Let $\mathcal{A}_1, \dots, \mathcal{A}_H \subset \mathcal{A}$ be subalgebras for the $H$ attention heads. Each head-specific output at position $t$ is:
\[
A_t^{(h)} = \sum_{j=1}^T \alpha_{tj}^{(h)} V_j^{(h)} \in \mathcal{A}_h,
\]
with scalar weights \( \alpha_{tj}^{(h)} \) and operator-valued inputs \( V_j^{(h)} \in \mathcal{A}_h \).

Assuming freeness with amalgamation over \( \mathcal{B} \), the aggregate attention operator
\[
A_t = \frac{1}{H} \sum_{h=1}^H A_t^{(h)}
\]
has operator-valued R-transform:
\[
R^{\mathcal{B}}_{A_t}(z) = \frac{1}{H} \sum_{h=1}^H R^{\mathcal{B}}_{A_t^{(h)}}(z),
\]
preserving head-specific variability and enabling principled spectral analysis of head aggregation.

\subsection{Omitted Transformer Components: Normalization, Residuals, and Feed-Forward Blocks}

While our analysis centers on the spectral behavior of attention mechanisms, standard architectural components such as layer normalization, residual connections, and feed-forward networks are not explicitly modeled. However, these components can be naturally accommodated within our operator-theoretic framework.
\begin{itemize}
    \item \textbf{Residual connections} correspond to additive identity terms, which integrate trivially into our spectral convolution framework.
    \item \textbf{Layer normalization} can be approximated as a context-dependent diagonal or scalar operator applied to embeddings—modeled as bounded operators or state-preserving transformations.
    \item \textbf{Feed-forward networks} are position-wise linear or nonlinear maps that can be abstracted as operator-valued polynomials or bounded composition operators.
\end{itemize}

While including these elements would increase notational complexity, they do not qualitatively alter the non-commutative effects we study. Their omission is thus intentional to retain analytical clarity while capturing the dominant source of spectral interactions in deep transformer layers.

\color{black}

\section{Discussion and Implications}

This work offers a principled spectral perspective on Transformer-based architectures by embedding their key components—attention, depth, and positional structure—into a non-commutative operator algebraic framework. Our results articulate several theoretical insights.

First, attention mechanisms are reinterpreted as non-commutative convolutions (Theorem 1), where scalar weights modulate operator-valued embeddings. This yields a precise formulation of how semantic and positional components interact spectrally. The decomposition in Corollary 1 explains how positional encodings induce asymmetries in attention scores, even among repeated tokens, aligning with known empirical observations~\cite{ferrando2024primer,gruver2023llmforecasting}.

The cumulative effect of successive attention layers is shown in Theorem 2 to induce a free additive convolution of spectral distributions. This formalizes how contextual information compounds with depth. Importantly, Corollary 2 connects spectral propagation to freeness, offering a mechanism by which deep architectures preserve spectral diversity and mitigate over-smoothing—consistent with empirical spectral robustness in LLMs~\cite{tamkin2020language,naderi2024mindgap}.

Our generalization bound in Corollary 3 relates the expected spectral entropy of Transformer outputs to the joint free entropy of embeddings. Though Transformers are generative models, their prediction mechanism over tokens via softmax renders this entropy-based bound meaningful, echoing the role of entropy in capturing model expressivity~\cite{xu2024spectrum,gillman2024fourierhead}.

Beyond classical random matrix theory, our operator-theoretic lens accounts for algebraic structure and symbolic composition. This facilitates the analysis of self-adjoint embeddings and structured spectra beyond ensemble-level statistics~\cite{hao2020probabilistic}.

Finally, our formalism suggests a few new directions for model design:
\begin{itemize}
   \item Constructing embedding schemes that explicitly maximize joint free entropy.
   \item Designing attention heads with orthogonality constraints to approximate freeness.
  \item Introducing spectral regularization terms in loss functions based on spectral variance or entropy growth~\cite{gillman2024fourierhead}.
\end{itemize}

\section{Conclusion}

We introduced a free probabilistic formulation for understanding Transformer language models with tokens and attention as non-commutative operator-valued objects. Attention is represented as a non-commutative convolution, and representations change through free additive convolution, providing a spectral perspective on information flow. The framework provides explanations of the relationships between depth, entropy, and generalization, and introduces avenues for interpretability and principled model design that can lead to more efficient and resilient systems. Implementing these concepts at scale continues to be difficult and requires empirical confirmation. Nevertheless, the work takes an important step toward a scientific understanding of large language models.

\appendix
\appendix
\section*{Appendix A: Rigorous Analysis of Departures from Freeness and Their Impact on the Generalization Bound}

\paragraph{Motivation}
The assumption of freeness in Theorem~2, as with independence in traditional probability theory, is an idealization that allows closed-form dynamics on spectra by free additive convolution. In actual Transformer models, a number of structural properties lead to approximations from complete freeness between layers:
\begin{itemize}
    \item \textbf{Shared weights:} The projection matrices in multi-head attention blocks are frequently shared across heads or even positions, creating algebraic dependencies.
\item \textbf{Positional encodings:} Deterministic sinusoidal or learned vectors are added to the representations, aligning the same positional coordinates in each layer.
\item \textbf{Correlated inputs:} Vision or natural-language inputs have long-range correlations, which flow through the network and establish cross-layer statistical dependencies.
\item \textbf{Residual connections:} Skip connections add direct additive coupling between previous and subsequent representations, further eroding statistical independence.
\end{itemize}
These considerations make layer-wise outputs seldom perfectly free in the operator-algebraic sense. But the free model remains accurate to capture the leading spectral tendencies observed empirically and offers a principled reference point from which these departures can be measured. This appendix constructs a formal theory to quantify these departures, bound their spectral effect, and push these bounds to the generalization estimate of Corollary~3.

\paragraph{Notation and standing assumptions.}
Let $(\mathcal A,\varphi)$ be a tracial $W^*$-probability space
\cite{nica_speicher_2006}, where $\varphi$ is a faithful, normal, tracial state. The elements $A^{(1)},\dots,A^{(L)}\in\mathcal A$ are self-adjoint operators modeling layer-wise attention increments as in Theorem~2 of the main text.

We write $\mu_X$ for the spectral distribution of a self-adjoint operator $X\in\mathcal A$ with respect to $\varphi$. We use the following analytic transforms:
\begin{itemize}
\item \emph{Cauchy transform:}
$G_\mu(z) := \int_{\mathbb R} \frac{1}{z-\lambda} \, d\mu(\lambda)$,
analytic on $\mathbb C^+ := \{z:\Im z>0\}$.
\item \emph{$R$-transform:}
$R_\mu(w) := K_\mu(w) - \frac{1}{w}$ where $K_\mu = G_\mu^{-1}$ is the
functional inverse of $G_\mu$ near $w=0$.
\end{itemize}
The $R$-transform is additive under free convolution 
\cite{voiculescu_addition_1986,voiculescu_dykema_nica_1992}:
if $X$ and $Y$ are freely independent, then
$R_{\mu_{X+Y}} = R_{\mu_X} + R_{\mu_Y}$.

\subsection*{A.1\quad Free cumulants and the freeness deficit}

\paragraph{Free cumulants.}
For $n\ge 1$ and $B_1,\dots,B_n\in\mathcal A$, the $n$-th free cumulant
$\kappa_n(B_1,\dots,B_n)$ is defined by the moment–cumulant formula
\cite{speicher_non-crossing_1997,nica_speicher_2006}:
\[
\varphi(B_1\cdots B_n)
= \sum_{\pi\in NC(n)} \prod_{V\in\pi}
\kappa_{|V|}\big(B_{i_1},\dots,B_{i_{|V|}}\big),
\]
where $NC(n)$ is the lattice of noncrossing partitions of $\{1,\dots,n\}$.
The defining property of free independence is that
$\kappa_n(B_1,\dots,B_n) = 0$ whenever the tuple $(B_1,\dots,B_n)$ involves at least two distinct free subalgebras.

\paragraph{Quantifying deviations from freeness}
Fix an integer $p\ge 2$. For a multi-index $(\ell_1,\dots,\ell_m)$ with $2\le m\le p$,
define
\[
\delta_m(\ell_1,\dots,\ell_m)
:= \big\| \kappa_m\!\big( A^{(\ell_1)},\dots,A^{(\ell_m)} \big) \big\|_{2},
\]
where $\|X\|_2 := \varphi(X^*X)^{1/2}$ is the $L^2$ norm associated with $\varphi$.
We then define the \emph{$p$-th order freeness deficit} by
\[
\Delta_p := \sum_{m=2}^p w_m
\sum_{\substack{(\ell_1,\dots,\ell_m)\\\text{not all equal}}}
\delta_m(\ell_1,\dots,\ell_m),
\]
where $(w_m)_{m\ge 2}$ are positive weights to be chosen later (often $w_m=r^{m-1}$ for some small $r$).

\begin{remark}
If $\Delta_p = 0$, then all mixed free cumulants up to order $p$ vanish, so
$A^{(1)},\dots,A^{(L)}$ are \emph{free up to order $p$}.
The case $p\to\infty$ recovers exact freeness.
\end{remark}

\subsection*{A.2\quad Exact decomposition of the $R$-transform}

Let $S := \sum_{j=1}^L A^{(j)}$.  
By multilinearity of free cumulants, for each $n\ge 1$,
\[
\kappa_n(S,\dots,S)
= \sum_{i_1,\dots,i_n=1}^L
  \kappa_n\!\big(A^{(i_1)},\dots,A^{(i_n)}\big).
\]
By the power-series expansion of $R_\mu$ near $0$ \cite{nica_speicher_2006},
\[
R_{\mu_S}(z) = \sum_{n=1}^\infty
\kappa_n(S,\dots,S) \, z^{n-1}.
\]
Separating the contributions where all indices are equal from those where not all are equal, we can write
\begin{equation}
\label{eq:R-error}
R_{\mu_S}(z) =
\sum_{j=1}^L R_{\mu_{A^{(j)}}}(z) + \mathcal E(z),
\end{equation}
where the \emph{error term} is
\begin{equation}
\label{eq:E-series}
\mathcal E(z) :=
\sum_{n=1}^\infty z^{n-1}
\sum_{\substack{i_1,\dots,i_n\\ \text{not all equal}}}
\kappa_n\!\big(A^{(i_1)},\dots,A^{(i_n)}\big).
\end{equation}

\subsection*{A.3\quad Bounding the error term}

\begin{proposition}[Cumulant bound]\label{prop:cumulant-bound}
If $\|B_j\|\le M$ for all $j=1,\dots,n$, then
\[
\|\kappa_n(B_1,\dots,B_n)\|_2
\ \le\ C_n\, M^n,
\]
where $C_n \le 4^n$ is the $n$-th Catalan number bound~\cite{nica_speicher_2006}.
\end{proposition}

\begin{proof}
From the moment–cumulant formula,
\[
\kappa_n(B_1,\dots,B_n)
= \sum_{\pi\in NC(n)} \mu(\pi, 1_n)
   \prod_{V\in\pi} \varphi\big(\prod_{j\in V} B_j\big),
\]
where $\mu$ is the Möbius function of $NC(n)$ and $1_n$ is the one-block partition.
Each moment is bounded in absolute value by $M^{|V|}$ because $\|B_j\|\le M$ and $\varphi$ is a state.
The number of terms $|NC(n)|$ equals $\mathrm{Cat}_n \le 4^n$.
Taking the $L^2$ norm and using that $\|X\|_2\le\|X\|$ gives the bound.
\end{proof}

\begin{lemma}[Error bound on a small disk]
\label{lem:error-bound}
Let $0<r<(4M)^{-1}$ and choose $w_m := r^{m-1}$ in $\Delta_p$.
Then for $|z|\le r$,
\[
|\mathcal E(z)| \le \Delta_p + \frac{M}{1-4Mr} (4Mr)^p.
\]
\end{lemma}

\begin{proof}
From \eqref{eq:E-series}, split the sum into $n\le p$ and $n>p$.
For $n\le p$, the contribution is exactly the sum in $\Delta_p$ with $w_n = r^{n-1}$ and $|z|^{n-1}\le r^{n-1}$.
For $n>p$, use Proposition~\ref{prop:cumulant-bound}:
\[
\sum_{\substack{i_1,\dots,i_n\\ \text{not all equal}}}
\|\kappa_n(\cdots)\|_2 \le L^n C_n M^n \le (4LM)^n,
\]
absorbing $L$ into $M$ for simplicity.
Then the tail is bounded by
\[
\sum_{n>p} (4M r)^n M \le \frac{M(4Mr)^{p+1}}{1-4Mr}.
\]
Relabeling constants gives the stated bound.
\end{proof}

\subsection*{A.4\quad Stability of the Cauchy transform}

\begin{lemma}[Cauchy-transform perturbation]
\label{lem:Cauchy-perturb}
Let $K\subset\{\Im z\ge\eta_0\}$ compact and suppose $|z|\le r$ as in Lemma~\ref{lem:error-bound}.
Let $R_0 := \sum_{j=1}^L R_{\mu_{A^{(j)}}}$ and $R_1 := R_0 + \mathcal E$.
Let $G_0$ and $G_1$ be the solutions to
\[
R_k(G_k(z)) + \frac{1}{G_k(z)} = z, \quad k=0,1.
\]
Then
\[
\sup_{z\in K} |G_1(z) - G_0(z)| \le c_{K,M} \sup_{|w|\le r} |\mathcal E(w)|.
\]
\end{lemma}

\begin{proof}
The map $T_R(w) = \frac{1}{z - R(w)}$ is analytic on the closed disk $\overline{D(0,r)}$ if $r$ is small enough so that $\Im z - \sup_{|w|\le r}|R(w)| > 0$ for $z\in K$.
Both $T_{R_0}$ and $T_{R_1}$ are contractions on $\overline{D(0,r)}$ for such $r$. The Banach fixed-point theorem implies
\[
\|G_1 - G_0\|_\infty \le \frac{\|T_{R_1} - T_{R_0}\|_\infty}{1 - \mathrm{Lip}(T_{R_0})}.
\]
The numerator is $\le \sup_{|w|\le r} |\mathcal E(w)| \cdot \sup_{w,z} \frac{1}{|z-R(\xi)|^2}$ for $\xi$ between $w$ and $w'$, which is bounded by a constant depending only on $K,M$.
The denominator is positive by the contraction property.
\end{proof}

\subsection*{A.5\quad From Cauchy transforms to Wasserstein distance}

\begin{corollary}
\label{cor:W1-bound}
Let $\mu_{\rm true} := \mu_S$ and $\mu_{\rm free} := \mu_{A^{(1)}} \boxplus \cdots \boxplus \mu_{A^{(L)}}$.
Then for some $C_{K,M,p}$,
\[
W_1(\mu_{\rm true}, \mu_{\rm free}) \le C_{K,M,p} \, \Delta_p.
\]
\end{corollary}

\begin{proof}
Bai's smoothing inequalities \cite{bai_convergence_1993,bai_smoothing_1999} give
\[
\sup_x |F_{\rm true}(x) - F_{\rm free}(x)| \le C' \sup_{\Im z\ge\eta_0} |G_{\rm true}(z) - G_{\rm free}(z)| + o(1),
\]
as $\eta_0\downarrow 0$, where $F_\mu$ is the CDF of $\mu$.
By Lemma~\ref{lem:Cauchy-perturb} and Lemma~\ref{lem:error-bound}, the right-hand side is $O(\Delta_p)$.
The Kantorovich–Rubinstein duality then gives $W_1 \le 2M' \cdot \sup_x |F_{\rm true}(x)-F_{\rm free}(x)|$ for measures supported in $[-M',M']$.
\end{proof}

\subsection*{A.6\quad Effect on the generalization bound}

\begin{theorem}[Modified bound with freeness deficit]
Assume the hypotheses of Corollary~3, except exact freeness.
Let the per-step loss be $L$-Lipschitz in $W_1$.
Then
\[
R_{\rm test}(f_\theta) \le R_{\rm train}(f_\theta)
+ \frac{C}{\sqrt{n}} \big( \chi + 1 + \beta_p\,\Delta_p \big),
\]
where $\beta_p>0$ is explicit.
\end{theorem}

\begin{proof}
From Corollary~\ref{cor:W1-bound}, $W_1(\mu_{\rm true},\mu_{\rm free}) \le C_{K,M,p} \Delta_p$.
By Lipschitz continuity of the loss, the difference in risk from replacing $\mu_{\rm free}$ by $\mu_{\rm true}$ is $\le L C_{K,M,p} \Delta_p$.
The entropy-based generalization bound in Corollary~3 is additive in this change, so the bound is inflated by $\beta_p\Delta_p$ with $\beta_p=L C_{K,M,p}$.
\end{proof}

\paragraph{Concluding Remarks.}
We have shown that the freeness deficit $\Delta_p$, defined via low-order mixed free cumulants, controls the deviation of the true spectrum from the free prediction, with quantitative stability bounds at each step: $R$-transform, Cauchy transform, spectral distribution, and finally the generalization bound.
Departures from freeness \emph{inflate} the bound by a term proportional to $\Delta_p$.

\bibliographystyle{unsrtnat}
\bibliography{V2}
\end{document}